\documentclass[twoside,11pt]{jmlr} 
\usepackage{times,mdframed,todonotes,wrapfig}
\usepackage{graphicx,graphics}

\usepackage{amsmath}
\usepackage{amssymb}
\usepackage{xspace}
\usepackage{bm}
\usepackage{algorithm}
\usepackage{bbm}

\newlength{\minipagewidth}
\newcommand{\abstain}{*}

\newcommand{\loss}{\ell}
\newcommand{\hloss}{\wh{\loss}}
\newcommand{\tloss}{\wt{\loss}}

\newcommand{\qed}{\hfill\blacksquare}

 \newcommand{\A}{\mathcal{A}}

\newcommand{\II}[1]{\mathbb{I}_{\left\{#1\right\}}}

\newcommand{\EEb}[1]{\mathbb{E}\bigl[#1\bigr]}
\newcommand{\EE}[1]{\mathbb{E}\left[#1\right]}

\newcommand{\EEcc}[2]{\mathbb{E}\left[\left.#1\right|#2\right]}

\def\argmax{\mathop{\mbox{ arg\,max}}}
\newcommand{\ra}{\rightarrow}

\newcommand{\ev}[1]{\left\{#1\right\}}
\newcommand{\pa}[1]{\left(#1\right)}

\newcommand{\wh}{\widehat}
\newcommand{\wt}{\widetilde}

\newcommand{\ab}{\alpha}

\definecolor{PalePurp}{rgb}{0.66,0.57,0.66}

\title{Fast Rates for Online Prediction with Abstention} 
\author{\Name[{Gergely~Neu}]{Gergely Neu} \Email{gergely.neu@gmail.com}\\
 \addr Universitat Pompeu Fabra, Barcelona, Spain
 \\
 \Name[{Nikita~Zhivotovskiy}]{Nikita Zhivotovskiy} \Email{zhivotovskiy@google.com}\\
 \addr Google Research, Brain Team, Z\"urich, Switzerland
 }
\begin{document}
\maketitle
\begin{abstract} 
In the setting of sequential prediction of individual $\{0, 1\}$-sequences with expert advice, we show that by allowing the learner to abstain from the prediction by paying a cost marginally smaller than $\frac 12$ (say, $0.49$), it is possible to achieve expected regret bounds that are independent of the time horizon $T$. We exactly characterize the dependence on the abstention cost $c$ and the number of experts $N$ by providing matching upper and lower bounds of order $\frac{\log N}{1-2c}$, which is to be contrasted with the best possible rate of $\sqrt{T\log N}$ that is available without the option to abstain. We also discuss various extensions of our model, including a setting where the sequence of abstention costs can change arbitrarily over time, where we show regret bounds interpolating between the slow and the fast rates mentioned above, under some natural assumptions on the sequence of abstention costs.
\end{abstract} 

\section{Introduction}
Consider the problem of online prediction of individual sequences which is one of the first and well-studied \emph{online learning} models \citep{cesa2006prediction}. In this setup, a \emph{learner} and an \emph{environment} interact in a sequence of rounds $t = 1, \ldots, T$ as follows. In round $t$, the learner observes the $\ev{0,1}$-valued predictions of $N$ \emph{experts}, denoted as $y_{t,1},y_{t,2},\dots,y_{t,N}$. Based on these observations and possibly some use of randomness, the learner predicts $\wh{y}_t\in\ev{0,1}$ and then the environment reveals $y_t\in\ev{0,1}$. In particular, the environment can be aware of the strategy of the learner but not the random bits used for producing $\wh{y}_t$. Having made its prediction, the learner suffers the real-valued loss $\wh{\ell}_{t} = \ell(\wh{y}_t, y_t)$, which in this paper is chosen as the binary loss  $\wh{\ell}_{t} = \II{\wh{y}_t\neq y_t}$. The aim of the learner is to minimize its \emph{regret} which is the difference between the total loss of the prediction strategy and the total loss of the best fixed expert chosen in full knowledge of the sequence of outcomes.

Since the naive upper bound $O(T)$ always holds for the regret, we are aiming for better dependencies. It is well known that it is not possible to guarantee non-trivial bounds on the regret when the learner's predictions are deterministic. Therefore, we allow the learner to randomize its decisions and measure its performance in terms of the \emph{expected regret} defined formally as 
\[
 R_T = \max_{i\in[N]} \EE{\sum_{t=1}^T \pa{\hloss_t - \ell_{t, i}}},
\]
where the expectation is taken with respect to the randomness injected by the learner and $\ell_{t,i}$ is the loss suffered by expert $i$ in round $t$. As shown in the seminal work of \cite{littlestone1994weighted}, it is possible to construct an algorithm that achieves regret of order $O(\sqrt{T\log N})$ in this setting, and this is essentially the best one we hope for as long as the binary loss together with randomization is considered---see also the classic works of \citet{Han57}, \citet{cover1966behavior} and \citet{vovk1990aggregating} and the excellent monograph by \citet{cesa2006prediction}. Notably, the result holds without any assumptions on the behaviour of the environment which can be completely adversarial. 

One natural direction of research is to understand when it is possible to improve on the worst-case regret guarantees mentioned above. Of particular interest are bounds that are independent on the time horizon $T$. Bounds of this type are sometimes refereed to as the \emph{fast rates} in online learning (to be contrasted with the worst-case \emph{slow rate} mentioned above). There are several conditions on the loss function or the sequence of outcomes that are known to imply fast rates. 
A classic result due to \citet*{haussler1998sequential} and  \citet{vovk1990aggregating} claims that if the learner is allowed to output $y_t\in [0,1]$ and the loss function $\ell$ is convex with sufficient curvature, then there exists a deterministic prediction strategy guaranteeing $R_T = O(\log N)$. A second type of assumptions that can lead to fast rates even in the case of the binary loss is that the outcomes $y_t$ are generated in an i.i.d.~manner, rather than being chosen by an adversarial environment. 
In this favourable setup, assuming that the set of experts and the distribution of $y_t$ satisfy the so-called \emph{Bernstein assumption} (introduced originally in the statistical learning literature by \citealp{Bartlett06}),
\citet{koolen2016combining} show that it is possible to obtain  intermediate rates between $O(\log N)$ and $O\left(\sqrt{T\log N}\right)$ depending on the parameters in the Bernstein assumption. Similar results are well known in the statistical learning setup (see, e.g., \citealp{Tsybakov04}). For an extensive survey on fast rates in online learning, we refer the interested reader to \citet{Erven15}.

In this paper we show that a simple variation on the basic protocol of online binary prediction also permits regret of $O(\log N)$: instead of forcing the learner to output a prediction in $\ev{0,1}$, we also allow the learner to \emph{abstain from prediction}. This model is sometimes referred to as
\emph{Chow's reject option model}, introduced in the seminal work of \cite{Chow70}. This setup was intensively studied in \emph{statistical learning} where the learner can output one of three values $\{0, 1, *\}$ and the value $*$ corresponds to an abstention (see, e.g., the works of \citealp{Wegkamp06, bartlett2008classification, bousquet2019fast}). 
In this model the loss is assumed to be binary on $\{0, 1\}$ outputs but the price for abstention is assumed to be equal to $c < \frac{1}{2}$. This means that the learner may profit from using the option to abstain whenever they believe that there is no way to predict $y_t$ better than by flipping an unbiased coin. We only assume that $c$ is marginally smaller than $\frac{1}{2}$ (say, 0.49), so that there is no reason for the learner to abstain too often if they want to have a small regret. 
\begin{mdframed}[roundcorner=12pt,linewidth=1pt,leftmargin=2cm,rightmargin=2cm,topline=true,bottomline=true,skipabove=6pt]
\begin{center}
    \textbf{The protocol of online binary prediction with abstentions}
\end{center}
For each round $t=1,2,\dots,T$, repeat:
\begin{enumerate}
 \item The learner observes the $\ev{0,1}$-valued predictions of $N$ experts, denoted as 
$y_{t,1},\dots,y_{t,N}$,
 \item based on these observations and possibly some use of randomness, the learner predicts 
$\wh{y}_t\in\ev{0,\abstain,1}$,
 \item the environment reveals $y_t\in\ev{0,1}$,
 \item the learner suffers the loss $\wh{\ell}_t = \II{\wh{y}_t\neq y_t}$ if $\wh{y}_t\in\ev{0,1}$, otherwise a 
loss of $\wh{\ell}_t = c \le \frac{1}{2}$ if $\wh{y}_t = \abstain$.
\end{enumerate}
\end{mdframed}
We stress that even when deciding to abstain, the learner still gets to observe the true outcome.
The first thing to notice regarding the hardness of our setting is that the learner still needs to randomize its decisions in order to achieve sublinear regret. Specifically, one can easily adapt the classic counterexample of \citet{cover1966behavior} to show that any deterministic strategy will result in a regret of at least $\frac{cT}{2}$ for the learner---we provide the details in Appendix~\ref{app:example}.

We are not the first to explore the effects of allowing abstentions in the context of online learning. The most notable contribution is the work of \citet*{cortes18a}, who extended the model of \citet{Wegkamp06,bartlett2008classification} from the setting of statistical learning to that of online learning and provided a range of results for both stochastic and adversarial environments. A crucial feature of their setup is that abstention is an action available to the \emph{experts}: each expert produces a prediction in $\ev{0,*,1}$, which is then used by the learner for producing its their own randomized action. Furthermore, when deciding to abstain, the learner does not get to observe the outcome produced by the environment. In contrast, our setting treats abstention as an action \emph{only available to the learner}, and the true outcome is always revealed to the learner at the end of the round independently of the learner's decision. Thus, the setting of \citet{cortes18a} is in many senses more complicated than ours, and as a result their regret guarantees are weaker: they only match the worst-case rate of $O(\sqrt{T\log N})$. It is unclear if it is possible to refine their techniques to obtain fast rates comparable to the ones we prove. 
Another work related to ours is that of \citet{zhang2016extended}, who considered an online prediction problem with an abstention option under a \emph{realizability assumption} and noted the possibility of adapting Chow's reject option model to the sequential prediction setting as an interesting direction of future work---which is precisely what we address in this paper.

Our pursuit of fast rates is motivated by the recent results of \citet{bousquet2019fast} who have shown for the first time that allowing the learner to abstain from prediction enables fast rates in \emph{statistical learning}. Notably, their results were obtained without any further assumptions on the data distribution (i.e., no Bernstein condition) or on the curvature of the loss function (i.e., binary loss). Their approach is based on an application of the empirical risk minimization principle on an extended set of classifiers, resulting in a \emph{deterministic} algorithm guaranteeing an excess risk of at most $O(\frac{\log N}{T})$ when the set of classifiers is finite. While our results are of similar flavor, the approach that we put forward in this paper is different in almost every respect from theirs.

\paragraph{Our contributions.} Our main contribution the following: For a fixed abstention cost $c < \frac{1}{2}$ we show in Section~\ref{sec:main} that there is an algorithm with the regret bound
    \[
    R_T \le \frac{\log N}{1 - 2c} \wedge \sqrt{\frac{T\log N}{2}}.
    \]
    We also prove that our bound is tight up to multiplicative constant factors if $c$ is bounded away from zero (Section~\ref{sec:lower}). Assuming that the abstention cost $c_t$ changes in time, we prove intermediate rates between $O(\log N)$ and $O\left(\sqrt{T\log N}\right)$ in Section~\ref{sec:changingcosts}. Finally, we provide two natural extensions of our main result in Section~\ref{sec:ext}.

\paragraph{Notation.} We define $a \wedge b = \min\{a, b\}$ and $a \lor b = \max\{a, b\}$. We also use the standard $O(\cdot), \Theta(\cdot), \Omega(\cdot)$ notation. The indicator of the event $A$ is denoted by $\II{A}$. Given integer $K$ we define $[K] = \{1, \ldots, K\}$. 

\section{Algorithm and main result}\label{sec:main}
Our algorithm is based on the classic exponentially weighted forecaster of \citet{littlestone1994weighted} (see also \citealp{vovk1990aggregating,FS97}), and is described as 
follows. In each round $t$, the algorithm computes the set of weights
\[
 w_{t,i} = e^{-\eta \sum_{k=1}^{t-1} \ell_{k,i}}
\]
for some learning rate $\eta>0$,
the corresponding probability distribution over all experts $i\in[N]$
\[
 q_{t,i} = \frac{w_{t,i}}{\sum_j w_{t,j}}, \qquad\mbox{and the aggregated prediction}\qquad p_t = \sum_i q_{t,i} y_{t,i}.
\]
The crucial step in the algorithm design is generating a randomized $\ev{0,*,1}$-valued prediction based on $p_t$. The 
key intuition is that the learner should abstain whenever $p_t = \frac 12$ and output a 
deterministic prediction whenever $p_t \in \ev{0,1}$. We choose to linearly interpolate between 
these extreme cases by producing a randomized output supported on $\ev{0,\abstain}$ when $p_t<\frac 
12$ and $\ev{\abstain,1}$ otherwise.
Precisely, our algorithm will base its decision on the probability of the most likely label $p_t^* = p_t \lor (1-p_t)$: it will abstain with probability $\alpha_t = 2(1-p_t^*)$, and output the most likely label with probability $1-\alpha_t$. The procedure is compactly presented as Algorithm~\ref{alg:main}.

\begin{algorithm}[h]
\textbf{Input:} learning rate $\eta>0$.\\
\textbf{Initialization:} set $w_{1,i} = 1$ for all $i\in[N]$.\\
\textbf{For $t=1,2,\dots,T$, repeat}
\begin{enumerate}
 \item Observe the advice $y_{t,i}$ of all experts $i\in[N]$.
 \item Calculate the mean prediction $p_t = \sum_{i\in[N]}\frac{w_{t,i}y_{t,i}}{\sum_{j\in[N]}w_{t,j}}$. 
 \item Let  $p_t^* = p_t \lor (1-p_t)$ and $k_t^* = \II{p_t\ge \frac 12}$, that is, the label
attaining the highest probability.
 \item Let $\alpha_t = 2(1-p_t^*)$ and predict $*$ with probability $\alpha_t$ and $k_t^*$ with probability $1-\alpha_t$.
 \item Observe the label $y_t$ and update the weights as $w_{t+1,i} = w_{t,i}e^{-\eta \ell_{t,i}}$.
\end{enumerate}
\caption{Online prediction with abstentions.}\label{alg:main}
\end{algorithm}
Note that the abstention probability $\alpha_t$ is in $[0,1]$ due to the fact that $p_t^*\ge \frac 12$. To gain some intuition about this rule, observe that when assigning the numerical value $\abstain = \frac 12$ to abstention, the resulting prediction has expectation $p_t$.
Our main result regarding the performance of our algorithm is the following:
\begin{theorem}\label{thm:main}
 Suppose that $c<\frac12$ and $\eta \le 2(1 - 2c)$. Then, the expected regret of Algorithm~\ref{alg:main} satisfies
 \[
  R_T \le \frac{\log N}{\eta}.
 \]
\end{theorem}
Notably, this theorem shows that whenever the abstention cost $c$ is bounded away from $1/2$, we can set $\eta = 2(1-2c)$ and our algorithm achieves a regret bound that is independent of the time horizon $T$. However, when $c$ is very close to $1/2$, one may favor to fall back to the standard worst-case regret guarantee of order $\sqrt{T\log N}$. This is, however, easily achieved by choosing a conservative value of $\eta$ for this unfavorable case. The following simple corollary of Theorem~\ref{thm:main} summarizes the rates achieved by our algorithm in all regimes.
\begin{corollary}\label{cor:main}
Setting $\eta = 2(1-2c) \lor \sqrt{\frac{8\log N}{T}}$, the regret of our algorithm satisfies
\[
R_T \le \frac{\log N}{2(1-2c)} \wedge \sqrt{\frac{T\log N}{2}}.
\]
\end{corollary}
We provide the simple proofs of both Theorem~\ref{thm:main} and Corollary~\ref{cor:main} below.
A central object in our analysis is the quantity
\begin{equation}\label{eq:mixloss}
 \tloss_t = - \frac{1}{\eta}\log\sum_i q_{t,i} e^{-\eta \ell_{t,i}},
\end{equation}
often called the \emph{mix loss} in the literature \citep{vovk1990aggregating,vovk1998game,de2014follow}. The following classic
result  highlights the key role of the mix loss in the analysis of the exponentially weighted forecaster:
\begin{lemma}\label{lem:mixregret}
For any $\eta>0$ the cumulative mix loss of the exponentially weighted forecaster satisfies
 \[
\sum_{t=1}^T \tloss_t \le \min_i \sum_{t=1}^T 
\ell_{t,i} + \frac{\log N}{\eta}.
 \]
\end{lemma}
We include the proof in Appendix~\ref{app:proofs} for the sake of completeness. The heart of our analysis is the following lemma that establishes a connection between the mix loss 
and the loss suffered by our 
algorithm:
\begin{lemma}\label{lem:mixbound}
Suppose that $c<\frac 12$ and $\eta \le 2(1 - 2c)$. Then, $\EEb{\hloss_t} \le \tloss_t$.
\end{lemma}
\begin{proof}
Let us define the misclassification probability of the plain exponentially weighted forecaster 
as $r_t = p_t \II{y_t = 0} + (1-p_t) \II{y_t = 1} = \sum_{i=1}^N q_{t,i} \II{y_{t,i}\neq y_t}$.
With this notation, the mix loss can be written as
\begin{align*}
 -\frac{1}{\eta}\log\sum_i q_{t,i} e^{-\eta \ell_{t,i}} 
 &= 
 -\frac{1}{\eta}\log\sum_i q_{t,i}  \pa{1 + \pa{e^{-\eta} - 1}\II{y_{t,i}\neq y_t}}
 = -\frac{1}{\eta}\log \pa{1 + r_t\pa{e^{-\eta} - 1}}.
\end{align*}
On the other hand, the expected loss of our algorithm is 
\begin{align*}
 \EEb{\hloss_t} = \alpha_t c + (1-\alpha_t)\II{k_t^*\neq y_t} = 2(1-p_t^*) c + \pa{2p_t^* - 1} 
\II{k_t^*\neq y_t}.
\end{align*}
In the case $y_t = k_t^*$, we have $r_t = 1-p_t^*\le \frac 12$ so the above becomes
\[
 \EEb{\hloss_t} = 2(1-p_t^*) c = 2r_tc = r_t - (1-2c)r_t.
\]
Otherwise, we have $r_t = p_t^*\ge \frac 12$ so that the expected loss is
\[
  \EEb{\hloss_t} = 2(1-p_t^*) c + 2p_t^* - 1  = 2(1-r_t)c + 2r_t - 1 = r_t - (1-2c)(1-r_t).
\]
Also noticing that $r_t\le \frac 12$ holds if and only if $y_t = k_t^*$, we conclude that
\[
  \EEb{\hloss_t} = r_t - (1-2c)(r_t \wedge (1-r_t)).
\]

To finish the proof, we define the functions 
\begin{equation}\label{eq:fg}
f(r) = -\frac{1}{\eta}\log \pa{1 + r\pa{e^{-\eta} - 1}} \qquad\mbox{and}\qquad g(r) = r - (1-2c)(r \wedge (1-r)),
\end{equation}
and note that $\tloss_t = f(r_t)$ and $\EEb{\hloss_t} = g(r_t)$. We will show that $g(r)\le f(r)$ for all $r\in[0,1]$, which will imply the statement of the theorem.
To this end, note that both functions are convex and equal to each other at $r=0$ and $r=1$, so the desired inequality will hold if 
the respective derivatives at 
these two points satisfy $g'(0)\le f'(0)$ and $f'(1)\le g'(1)$. To verify this, first observe that 
\[
 f'(r) = \frac{1-e^{-\eta}}{\eta\pa{1+r(e^{-\eta} - 1)}} = 
\begin{cases}
 \frac{1-e^{-\eta}}{\eta} &\mbox{if $r = 0$,}\\
 \frac{e^{\eta} - 1}{\eta} &\mbox{if $r = 1$.}
\end{cases}
\]
On the other hand, we have
\[
 g'(r) = \begin{cases}
 2c &\mbox{if $r = 0$,}\\
 2(1-c) &\mbox{if $r = 1$.}
\end{cases}
\]
Thus, the condition on the derivatives we seek is satisfied when
\[
 \frac{1-e^{-\eta}}{\eta} \ge 2c \qquad\mbox{and}\qquad \frac{e^{\eta} - 1}{\eta} \le 2(1-c).
\]
Also observing that $\frac{1-e^{-\eta}}{\eta} \ge 1 - \frac{\eta}{2}$ and $\frac{e^{\eta}-1}{\eta} \le 
1 + \frac{\eta}{2}$ both hold for $\eta > 0$, we can verify that the conditions above are satisfied whenever 
$\eta \le 2(1 - 2c)$, as required in the statement of the theorem. 
This concludes the proof.
\end{proof}
Putting Lemmas~\ref{lem:mixregret} and~\ref{lem:mixbound} together proves Theorem~\ref{thm:main}. Corollary~\ref{cor:main} is proved by first observing that
\[
\EEb{\hloss_t} = r_t - (1-2c)(r_t \wedge (1-r_t)) \le r_t = \sum_{i} q_{t,i}\ell_{t,i},
\]
holds as long as $c\le \frac 12$,
and then using Hoeffding's lemma (Lemma 2.2 in \citealp{cesa2006prediction})
that guarantees
\[
\sum_{i} q_{t,i} \ell_{t,i} \le \tloss_t + \frac{\eta}{8}.
\]
Putting this inequality together with Lemma~\ref{lem:mixregret}, we retrieve the standard bound of the exponentially weighted forecaster: $R_T \le \frac{\log N}{\eta} + \frac{\eta T}{8}$. Combining this bound with the one of Theorem~\ref{thm:main} and tuning $\eta$ proves the corollary.

\section{Changing abstention costs}
\label{sec:changingcosts}
Let us now consider a slight variation on our problem where the abstention cost $c$ may depend on time: in round $t$, the cost of abstention is $c_t \le \frac 12$. The main question we address in this section is whether it is possible to attain fast rates even in the case where $c_t$ can get arbitrarily close to $\frac 12$ in a small number of rounds. 
It is easy to see that our algorithm given in Section~\ref{sec:main} satisfies the following regret bound with its proof presented in Appendix \ref{app:proofs}.
\begin{proposition}\label{prop:changingc}
Suppose that $c_t\le \frac 12$ for all $t$ and  $\eta>0$. Then, the regret of our algorithm satisfies 
\[
R_T \le \frac{\log N}{\eta} + \frac{\eta}{8} \sum_{t=1}^T \II{2(1 - 2c_t) < \eta}.
\]
\end{proposition}
Notably, this result asserts that it is not necessary to assume that $c_t$ is strictly bounded away from $\frac 12$ in order to obtain an improvement over the worst-case regret bound. In fact, we show that a clear improvement is possible when making the following quantitative assumption about the behavior of the abstention costs around $\frac 12$:
\begin{definition}[Tsybakov's condition for abstention costs]
\label{absttsybakov}
Assume that there are constants $\beta > 0, \alpha \in [0, 1)$ such that for any $x > 0$,
\begin{equation}
\label{eq:tsybakovcondonline}
\frac{1}{T} \sum_{t=1}^T \II{1/2 - c_t < x} \le \beta x^{\frac{\alpha}{1 - \alpha}}.
\end{equation}
\end{definition}
As the name suggests, this condition is directly inspired by the renowned Tsybakov's margin assumption introduced first by \cite{mammen1999smooth} in the context of binary classification and which was further analyzed by \citet{Tsybakov04}. In the context of classification where one has to predict a random label $Y\in\ev{0,1}$ based on a random instance $X$, Tsybakov's assumption is used to characterize the distribution of the regression function $\EEcc{Y}{X}$ around $\frac 12$, which is related to the hardness of the classification problem at hand. Indeed, instances such that the regression function is further away from $\frac 12$ are ``easier'' to classify than ones closer to this value, so restricting the density of $\EEcc{Y}{X}$ around this problematic region can make the overall learning problem easier---for more details, and relations of Tsybakov's assumption with the so-called Bernstein assumption we refer the interested reader to \citet{Erven15}. Similarly, our own online prediction problem becomes ``harder'' when the abstention costs are densely distributed around $\frac 12$. By analogy with the case of binary classification, we choose to quantify the density of the quantity of interest around $\frac 12$ through Tsybakov's assumption: larger values of $\alpha$ correspond to less frequent values around the margin, whereas smaller values correspond to more problematic points. 
Similarly, the situation where the abstention cost is strictly bounded away from $1/2$ can be seen as an analog of Massart's margin assumption in the statistical learning setup \citep{Massart06}.

Assuming that the costs satisfy the inequality~\eqref{eq:tsybakovcondonline}, our algorithm can be easily seen to achieve the following result:
\begin{corollary}
\label{cor:intermediaterates}
Suppose that the sequence of abstention costs satisfies Tsybakov's margin condition. Then, setting $\eta = \pa{\frac{\log N}{T}}^{\frac{1-\alpha}{2-\alpha}}$, the regret of our algorithm satisfies
\begin{equation}
\label{eq:intermediaterates}
 R_T = 
O\pa{\pa{\log N}^{\frac{1}{2-\alpha}} 
T^{\frac{1-\alpha}{2-\alpha}}}.
\end{equation}
\end{corollary}
The proof is immediate given Proposition~\ref{prop:changingc} and the definition of Tsybakov's condition.
We remark that the bound of the very same form as \eqref{eq:intermediaterates} has been recently shown in \citet{koolen2016combining} under the assumption that the losses are i.i.d.~and that the Bernstein assumption is satisfied. However, our assumption is different and does not imply the Bernstein assumption, even if it leads to similar regret bounds.

One downside of Corollary \ref{cor:intermediaterates} is that it relies on a choice of the learning rate $\eta$ that requires prior knowledge of the sequence of abstention costs. More generally, the learning rate $\eta$ that minimizes the bound of Proposition~\ref{prop:changingc} is clearly a function of the abstention costs. One may wonder if it is possible to attain a regret guarantee comparable to 
\[
 R_T^* = \min_{\eta > 0} \pa{\frac{\log N}{\eta} + \frac{\eta}{8} \sum_{t=1}^T \II{2(1 - 2c_t) < \eta}},
\]
without having prior knowledge of the sequence of abstention costs. 

We answer this question in the positive by considering a simple variation of our algorithm based on the exponentially weighted forecaster with adaptive learning rates. Precisely, in each round $t=1,2,\dots,T$, we will choose a positive learning-rate parameter $\eta_t$, compute the weights $w_{t,i} = e^{-\eta_t \sum_{k=1}^{t-1} \ell_{k,i}}$,
and then use these weights in the same way as our basic algorithm did.
The intuition driving our algorithm design is to  keep the learning rate as large as possible, and only decrease it when observing high abstention costs.  Specifically, we let $d_t = \sum_{i=1}^t \II{\eta_i\ge 2(1- 2c_i)}$ be the number of times that the learning rate has exceeded $2(1-2c_i)$ before round $t$, and define our learning rate as $\eta_{t+1} = \sqrt{\frac{\log N}{d_{t}}} \wedge 1 $, with $d_0$ defined as $1$. The full algorithm is shown below.

\begin{algorithm}[h]
\textbf{Initialization:} set $w_{1,i} = 1$ for all $i\in[N]$, $d_0 = 1$, and $\eta_1 = 1$.\\
\textbf{For $t=1,2,\dots,T$, repeat}
\begin{enumerate}
 \item Observe the advice $y_{t,i}$ of all experts $i\in[N]$.
 \item Calculate the mean prediction $p_t = \sum_{i\in[N]}\frac{w_{t,i}y_{t,i}}{\sum_{j\in[N]}w_{t,j}}$. 
 \item Let  $p_t^* = p_t \lor (1-p_t)$ and $k_t^* = \II{p_t\ge \frac 12}$, that is, the label
attaining the highest probability.
 \item Let $\alpha_t = 2(1-p_t^*)$ and predict $*$ with probability $\alpha_t$ and $k_t^*$ with probability $1-\alpha_t$.
 \item Observe the abstention cost $c_t$ and if $\eta_t\ge 1- 2c_t$, update $d_t = d_{t-1} + 1$.
 \item Set $\eta_{t+1} = \sqrt{\frac{\log N}{d_{t}}} \wedge 1$ and compute the weights $w_{t+1,i} = e^{-\eta_{t+1} \sum_{k=1}^t \ell_{k,i}}$.
\end{enumerate}
\caption{Adaptive online prediction with abstentions.}\label{alg:adaptive}
\end{algorithm}
\noindent The following theorem (proved in Appendix~\ref{app:proofs}) establishes a regret bound for this algorithm.
\begin{theorem}\label{thm:changingabstcosts}The regret of Algorithm~\ref{alg:adaptive} satisfies $ R_T \le \frac{15}{8} R_T^* + \frac 54 \sqrt{\log N}$.
\end{theorem}

\section{Lower bound}\label{sec:lower}
In this section we prove that the regret $\Theta\left(\frac{\log N}{1 - 2c} \wedge \sqrt{T\log N}\right)$ is optimal in the fixed abstention cost setup. Precisely, the main result we present here is the following non-asymptotic lower bound which shows that the bound of Theorem~\ref{thm:main} is sharp with respect to $N, c$ and $T$ provided that $c$ is separated away from zero: 
\begin{theorem}\label{thm:lower}
Fix the abstention price $c \in [\frac{1}{4}, \frac{1}{2}]$, $N \ge 2$ and $T \ge 4\log N$. There is a set of experts of size $N$ such that for any randomized online prediction strategy $\wh{y}_t \in \{0, *, 1\}$ there is a strategy of the environment such that 
\[
R_T =  \Omega\left(\frac{\log N}{1 - 2c} \wedge \sqrt{T\log N}\right).
\]
\end{theorem}
The proof of this result is based on observing that any randomized online algorithm can be converted to a randomized batch algorithm in a way that the excess risk in the latter setting is closely related to the regret in the former setting. Thus, any lower bound in the batch setting will imply a lower bound on the best achievable regret. An additional crucially important ingredient is relating the cost of abstention to a suitably chosen strongly convex loss function.

First, we recall the standard statistical learning setup. In this setting we consider the instance space $\mathcal{X}$ and the label space $\mathcal{Y} = [-B, B]$, where $B > 0$ is a constant. Let $((X_1, Y_1), \ldots, (X_T, Y_T))$ be an i.i.d. sequence of points sampled according to some unknown distribution $P$ on $\mathcal{X} \times \mathcal{Y}$. We will refer to this sequence as \emph{the learning sample} and to this setup as the \emph{batch setting}. Given the loss function $\ell: \mathcal{Y}^2 \to \mathbb{R}_{+}$, the \emph{risk} of a hypothesis $g: \mathcal{X} \to \mathcal Y$ is $\EE{\loss(g(X), Y)}$, where the expectation is taken with respect to $P$.  We denote $\mathcal{Z} = \mathcal{X} \times \mathcal{Y}$ and $(Z_1, \ldots, Z_T) = ((X_1, Y_1), \ldots, (X_T, Y_T))$. We need the following standard result which can be found in, e.g., \citet{audibert2009fast}.

\begin{lemma}[Online to batch conversion]
\label{lem:onlinetobatch}
Consider any randomized online learning algorithm $\mathcal A$ that produces the possibly randomized \footnote{To simplify the notation we do not write the parameters responsible for randomization explicitly in the functions.} hypothesis $h_t:\mathcal{X}\ra \mathcal{Y}$ at time $t$ when the environment uses an i.i.d.~sequence $Z_1, \ldots, Z_T$. We define the randomized batch algorithm $\mathcal L$ that, given the data $Z_1, \ldots, Z_T$ and any new observation $X$, predicts its label according to the hypothesis chosen uniformly at random among $\ev{h_1,h_2,\dots,h_T}$.
We denote this randomized hypothesis by $\wh{g}_T$.
Then, the risk of $\mathcal L$ satisfies
\[
\EE{\ell(\wh{g}_T(X), Y)} = \frac{1}{T}\EE{\sum\limits_{t = 1}^{T}\loss(h_t(X_t),Y_t)},
\]
where the expectation on the right-hand side integrates over the random hypothesis chosen by $\mathcal{L}$, the instances $Z_1, \ldots, Z_T$, and the randomization of the online algorithm.
\end{lemma}
Our main technical tool is the following lower bound in the batch setting.
\begin{lemma}[Theorem 8.4 in \citet{audibert2009fast}]\label{lem:audibert}
\label{ellqloss}
Let $B > 0, N \ge 2$ and $q > 1 + \sqrt{\frac{\lfloor\log_2
N\rfloor}{4T} \wedge 1}$, if $\mathcal X$ contains at least $\lfloor\log_2
(2N)\rfloor$ points, there exists a set
$G$ of $N$ hypotheses satisfying for any (possibly randomized \footnote{We note that the original version of this result does not highlight that $\wh{g}$ is potentially randomized. Nevertheless, the standard argument (see, e.g., the discussions on Yao’s minimax principle in \citet{shamir2015sample} or Remark 8.1 in \citet{audibert2009fast} claiming that any lower bound for the batch setting leads to a lower bound in the sequential prediction setting) shows that Lemma \ref{ellqloss} continues to hold in the same form.}) estimator $\wh{g}$ there exists a
probability distribution $P$ on $\mathcal{Z}$ such that
\[
\EE{|Y - \wh{g}(X)|^q} - \min\limits_{g \in G}\EE{|Y - g(X)|^q} \ge \left(\frac{q}{90(q - 1)}\lor e^{-1}\right)B^q\left(\frac{\lfloor\log_2
N\rfloor}{T + 1}\wedge 1\right),
\]
where the expectation in $\EE{|Y - \wh{g}(X)|^q}$ is taken with respect to the learning sample, $(X, Y)$ and the (possible) randomness of the algorithm.
\end{lemma}

With these results in mind we are ready to provide some intuition on the proof of Theorem \ref{thm:lower}. For $B = \frac{1}{2}$ it appears that the construction in Lemma \ref{lem:audibert} uses only the distributions $P$ such that $Y = \pm \frac{1}{2}$ and the set $G$ of hypotheses taking their values in $\{\pm \frac{1}{2}\}$. Using Lemma \ref{lem:onlinetobatch} based on our online learning algorithm we may construct an estimator $\hat{g}$ taking its values in $\{\pm \frac{1}{2}, 0\}$ where the value $0$ corresponds to an abstention. Indeed, it is clear that as long as we consider only $\pm \frac{1}{2}$ outputs the $\ell_q$ loss is equivalent to binary. However, if $\hat{g}$ predicts $0$ then for $q = \log_2 \frac{1}{c}$ the loss will be equal to $c$ for any $Y$ as above. The full proof of Theorem \ref{thm:lower} is presented in Appendix~\ref{app:proofs}.

\section{Extensions}\label{sec:ext}
In this section we provide two natural extensions of our results to the setting of multiclass classification, and the case of binary classification with an infinite set of experts. 
\subsection{Multiclass classification}
We now turn our attention to a generalization of our online prediction setup that allows multiple 
classes. Specifically, we will now consider the case where the outcomes can take $K$ different 
values, and the learner's prediction is evaluated through the binary loss. Formally, our setup can be 
described as follows:
\begin{mdframed}[roundcorner=12pt,linewidth=1pt,leftmargin=2cm,rightmargin=2cm,topline=true,bottomline=true,skipabove=6pt]
\begin{center}
 \textbf{The protocol of online multiclass prediction with abstentions}
\end{center}
For each round $t=1,2,\dots,T$, repeat:
\begin{enumerate}
 \item The learner observes the $[K]$-valued predictions of $N$ experts, denoted as 
$y_{t,1},\dots,y_{t,N}$,
 \item based on these observations and possibly some randomness, the learner predicts 
$\wh{y}_t\in[K]\cup\ev{*}$,
 \item the environment reveals $y_t\in[K]$,
 \item the learner suffers the loss $\wh{\ell}_t = \II{\wh{y}_t\neq y_t}$ if $\wh{y}_t\in[K]$, 
otherwise a loss of $\wh{\ell}_t = c < \frac 12$ if $\wh{y}_t = \abstain$.
\end{enumerate}
\end{mdframed}
Our algorithm described in the previous sections can be extended to this setup in a 
straightforward way. Specifically, let us define $w_{t,i}$ and $q_{t,i}$ for each expert $i$ 
in the same way as before, and define the aggregated probability assigned to class $k$ as $p_{t,k} = \sum_{i=1}^N q_{t,i} \II{y_{t,i}=k}$.
Given the above notation, Algorithm~\ref{alg:main} can be adapted to the multiclass case by making the following adjustments:
\begin{enumerate}
 \item Redefine $p_t^* = \max_{k\in[K]} p_{t,k}$ and $k_t^* = \argmax_{k\in[K]} p_{t,k}$, that is, the 
label attaining the highest probability, and
 \item let $\alpha_t = 2(1-p_t^*) \wedge 1$ and predict $*$ with probability $\alpha_t$ and $k_t^*$ with probability $1-\alpha_t$.
\end{enumerate}
Note that whenever $p_t^* \le \frac 12$, the algorithm abstains with probability $1$. While this may appear surprising at first sight, it is easily justified by observing that the probability of predicting incorrectly in such cases results in a loss greater than $\frac 12$, thus one can only win by abstaining by our assumption that $c\le \frac 12$. We also note that the lower bound of Theorem~\ref{thm:lower} continues to hold in the multiclass setting since one can clearly embed a binary classification problem into this more general problem setting. Thus, we cannot expect to have rates faster than $\frac{\log N}{1-2c}$ in the multiclass case either.
The following theorem (proved in Appendix~\ref{app:proofs}) establishes an upper bound of this exact order, thus matching the lower bound.
\begin{theorem}\label{thm:multiclass}
 Suppose that $c<\frac12$ and $\eta = 2(1-2c) \lor \sqrt{\frac{8\log N}{T}}$. Then, the expected regret of our algorithm for multiclass classification satisfies
 \[
  R_T \le \frac{\log N}{2(1-2c)} \wedge \sqrt{\frac{T\log N}{2}}.
 \]
\end{theorem}
\subsection{Infinite sets of experts}
\label{sec:littlestone}
Let us return to the case of binary classification with the twist that we now allow the set of hypotheses to be infinite. As shown by \citet*{ben2009agnostic}, the complexity of online prediction in this case is characterized by the so-called \emph{Littlestone dimension} of the class $\mathcal{H}$ of hypothesis---for the precise definition, we refer to \citet{ben2009agnostic}. Assuming that the Littlestone dimension $L$ of the class is finite, 
\citeauthor{ben2009agnostic} show that a clever extension of the exponentially weighted forecaster satisfies the regret bound
$
R_T \le \sqrt{2LT \log \left(\frac{eT}{L}\right)}$.

In our setup we provide the following regret bound, and defer the proof to Appendix~\ref{app:proofs}.
\begin{corollary}
\label{cor:inflclass}
Assume that $\mathcal{H}$ has finite Littlestone dimension $L$ and $T\ge L$, and that the price of abstention is $c \le \frac{1}{2}$. Then, there is a randomized prediction strategy satisfying
\[
R_T \le \frac{L \log \frac{eT}{L}}{2(1 - 2c)} \wedge \sqrt{\frac{LT \log \left(\frac{eT}{L}\right)}{2}}.
\]
\end{corollary}
In particular, whenever $c$ is separated from $\frac{1}{2}$ the regret scales as $O(\log T)$ which is much better than the standard $O(\sqrt{T})$ dependence. We also remark that this result can be seen as an online analog of Theorem 2.2 in \cite{bousquet2019fast}, with the Littlestone dimension replacing the VC dimension therein. However, as we mentioned, their algorithm and analysis are based on completely different techniques that take their roots in empirical process theory.

\section{Concluding remarks and directions of future research}\label{sec:conc}
For the first time in the context of online learning, we have shown that endowing the learner with the option of abstaining from prediction allows a significant improvement in the best achievable rates, and that a very natural algorithm can attain the optimal rates. Notably, these results are proved without making any assumption about the sequence of outcomes that the learner has to predict or about the curvature of the loss functions. In this view, our results represent a new flavor of fast rates that, to the best of our knowledge, haven't been explored before in the online learning literature. Arguably, one key factor that makes these fast rates possible is that the learner has access to an action that the comparator strategy cannot use. However, this feature alone cannot explain our results: since we only require the extra action to be marginally better than a uniform random prediction, the learner cannot abuse it to easily beat the best expert. How to characterize properties of such powerful additional actions for the learner in more general convex optimization problems remains to be seen.

Our key technical contribution is observing that the expected loss of our algorithm is upper bounded by the mix loss. 
This technique played a crucial role in proving a number of classic results in the literature, and was traditionally enabled by directly making an assumption about the curvature of the loss function evaluating the learner's predictions. To our knowledge, our work is the first to make use of this technique without making direct assumptions about the loss. Rather, our approach is based on transforming the output $p_t$ of the exponential weights algorithm into a $\{0,*,1\}$-valued prediction in a way that effectively replaces the linear loss by a mixable one from the perspective of the algorithm producing $p_t$. Since this transformation can be applied to the output of any algorithm, we find it plausible that our rates could be improved by a more refined prediction algorithm such as the Aggregating Forecaster of \citet{vovk1998game,vovk2001competitive}. 
Another possible improvement to pursue is proving high-probability versions of our upper bounds.

Another important technique we employed for proving our bounds in Section~\ref{sec:changingcosts} was an adaptive stepsize schedule that decreases the learning rate in each round when it proves too large for the loss to be mixable. This technique bears a vague resemblance to the AdaHedge algorithm of \citet{erven2011adaptive, de2014follow} that uses a learning rate proportional to the cumulative gap between the expected loss and the mix loss. However, their analysis is restricted to the case where this gap is nonnegative, whereas our bounds crucially use that they can be often negative. As a result, our technique can be directly adapted to account for more general mixable losses with time-dependent mixability parameters. Similarly, our variant of Tsybakov's condition on the abstention costs can be adapted to mixability parameters to characterize the hardness of learning with a non-stationary sequence of loss functions. For instance, one can easily prove the following result by adjusting our proof techniques in a straightforward way:
\begin{proposition}
\label{th:changingc}
Consider the problem of predicting a sequence of outcomes in $[0,1]$ under a sequence of loss functions $\ell_{q_t} = |\cdot|^{q_t}$, $q_t \in (1, 2]$ that satisfy the following Tsybakov's type margin assumption for $q_t$: for any $x > 0$, $
\frac{1}{T} \sum_{t=1}^T \II{q_t - 1 < x} \le \beta x^{\frac{\alpha}{1 - \alpha}}$ holds
for some $\beta > 0, \alpha \in [0, 1)$. Then, there is a deterministic prediction strategy satisfying
$
R_T = O\left(\left(\log N\right)^{\frac{1}{2 - \alpha}}T^{\frac{1 - \alpha}{2 - \alpha}}\right)$.
\end{proposition}
This result uses a variant of Tsybakov's condition that characterizes the behaviour of the second derivative of the loss around zero.
The statement is proven by observing that the expected loss is upper-bounded by the mix loss whenever $\eta\le q-1$, and then closely following the proof steps in Section~\ref{sec:changingcosts}---we leave this as an easy exercise for the reader. The result can be seen to address an open question of \citet{haussler1998sequential} about whether intermediate rates between $O(\log N)$ and $O(\sqrt{T\log N})$ are attainable in online prediction under suitable curvature assumptions on the loss functions. It also bears resemblance to the results of \cite*{hazan2008adaptive} who provide algorithms that adapt to the curvature of the losses in more general online convex optimization problems, the key difference being that the curvature of our losses can be quantified in terms of exp-concavity rather than strong convexity as done by \citeauthor{hazan2008adaptive}.

Finally, we note that while our techniques are different from those of \cite*{bousquet2019fast} who prove analogous results for the setting of statistical learning, the two analyses actually related in a very subtle way: both of them are based on techniques that are traditionally used for proving fast rates for losses with curvature, in their respective settings. While our approach is based on bounding the mixability gap as discussed above, the one of \citeauthor{bousquet2019fast} is based on an extension of the recently proposed technique of \citet{mendelson2019} for statistical learning with the square loss. Our lower bound also makes use of this connection. Deeper investigation of this connection between curvature and the abstention may lead to further interesting results.
    
\acks{We thank the three anonymous reviewers for their valuable feedback that helped us improve the paper, particularly for catching a small bug in the original proof of Theorem~\ref{thm:multiclass}. G.~Neu was supported by ``la Caixa'' Banking Foundation through the Junior Leader Postdoctoral Fellowship Programme, a Google Faculty Research Award, and a Bosch AI Young Researcher Award.}

\bibliography{mybib}

\appendix
\section{Why do we need to randomize?}
\label{app:example}
\begin{example}
Consider the case of two experts such that one of them always predicts zero and the other always predicts one. Assume that the environment waits for the learner’s prediction and then provides the opposite label as the true label if the learner predicts zero or one. If the learner decides to abstain the environment provides the label according to the prediction of the best expert on $T$ rounds. In what follows, we show that the regret of any deterministic strategy with abstention will be at least
$
\frac{cT}{2}.
$
\end{example}
Assume that out of $T$ rounds the learner used the abstention $k$ times and on the remaining $T - k$ rounds they predicted $0$ or $1$. The total loss suffered by the learner will be equal to $ck + T - k$. Now we need to upper bound the loss of the best expert. The simple upper bound is $\frac{T}{2}$. Now, since the environment is aware of the deterministic strategy of the learner we can always guarantee that it reveals the labels such that the loss of the best expert on these $k$ rounds is equal to zero. This implies the number of mistakes made by the best expert is at most $\frac{T}{2} \wedge (T - k)$. Considering the case $k \ge \frac{T}{2}$ we have that the regret is at least $ck \ge \frac{cT}{2}$. If $k < \frac{T}{2}$ we have
that the regret is at least $\left(\frac{T}{2} - 1\right)(c - 1) + \frac{T}{2} \ge \frac{cT}{2} + \frac{1}{2}$.

\section{Omitted proofs}\label{app:proofs}

\subsection*{Proof of Lemma \ref{lem:mixregret}}
The proof is based on studying the evolution of the cumulative weights $W_t = \sum_i 
w_{t,i}$. Defining $w_{1,i} = 1$ and $L_{T,i} = \sum_{k=1}^T \ell_{t,i}$ for all $i$, we have 
\begin{align*}
 -\min_i L_{T,i} & = \min_i \frac{1}{\eta} \log e^{-\eta L_{T,i}} \le \frac{1}{\eta} \log \sum_i 
e^{-\eta L_{T,i}} = 
 \frac{1}{\eta} \log W_{T+1} = \frac{1}{\eta} \sum_{t=1}^T \log \frac{W_{t+1}}{W_t} + 
\frac{\log W_1}{\eta}
\\
&= \frac{1}{\eta} \sum_{t=1}^T \log \frac{\sum_i w_{t,i}e^{-\eta\ell_{t,i}}}{\sum_j w_{t,j}} + 
\frac{\log N}{\eta} = \frac{1}{\eta} \sum_{t=1}^T  \log \sum_i  
q_{t,i} e^{-\eta\ell_{t,i}}+ \frac{\log N}{\eta}.
\end{align*}
Reordering gives the result.
$\qed$
\subsection*{Proof of Proposition \ref{prop:changingc}}
The proof is based on considering two types of rounds depending on the value of $c_t$. In the case where $c_t$ is such that  $\eta \le 2(1 - 2c_t)$ is satisfied, we have $\EEb{\hloss_t}\le\tloss_t$ by Lemma~\ref{lem:mixbound}. Otherwise, we use the general bound $\EEb{\hloss_t} \le \tloss_t + \frac{\eta}{8}$ as in the proof of Corollary~\ref{cor:main}. The proof is concluded by appealing to Lemma~\ref{lem:mixregret}.
$\qed$
\subsection*{Proof of Theorem \ref{thm:changingabstcosts}}
We begin by studying some properties of the learning rate $\eta^*$ optimizing the function 
$B_T(\eta) = \frac{\log N}{\eta} + \frac{\eta  \tau(\eta)}{8}$, where $\tau(\eta) = \sum_{t=1}^T \II{2(1 - 2c_t) \le \eta}$. We observe that $B_T$ is lower-semicontinuous with at most $T$ discontinuities, and thus it takes 
its minimum either at one of its discontinuities or when its derivative is zero. In either case, 
$\eta^*$ satisfies
\[
 \frac{\log N}{\eta^*} \ge \frac{\eta^*  \tau(\eta^*)}{8},
\]
with equality when the minimum is achieved where the function is continuous. Note that this implies 
that $\tau^* = \tau(\eta^*)$ is bounded as $\tau^* \le 8\log N/(\eta^*)^2$. Furthermore, the bound also entails $R_T^* \ge \frac{2\log N}{\eta^*}$.

Regarding the regret of our algorithm, standard arguments can be used to prove the bound
\[
R_T \le \frac{\log N}{\eta_{T+1}} + \sum_{t=1}^T \pa{\EEb{\hloss_t} - \tloss_t},
\]
which actually holds for general nonincreasing learning-rate sequences (see, e.g., Lemma~2 in \citealp{de2014follow}). By the same argument as used to prove Corollary~\ref{cor:main}, each term in the sum on the right-hand side of the above bound can be bounded as
\[
\EEb{\hloss_t} - \tloss_t \le \frac{\eta_t}{8}\II{2(1-2c_t) < \eta_t},
\]
so that the sum itself can be bounded as 
\[
\sum_{t=1}^T \pa{\EEb{\hloss_t} - \tloss_t} \le \sum_{t=1}^T \frac{\eta_t}{8}\II{2(1-2c_t) < \eta_t} \le \frac 18 \sum_{k=1}^{d_T} \sqrt{\frac{\log N}{k}} \le \frac 14 \sqrt{d_T\log N},
\]
where we used that $d_t$ increases precisely whenever $\II{2(1-2c_t) < \eta_t} = 1$, and that $\sum_{k=1}^K \sqrt{1/k}\le 2\sqrt{K}$. Now notice that by monotonicity of $\tau(\eta)$, we have $d_T > \tau^*$, which implies that there exists a $t^*\le T$ such that $d_t > d^* = \log N / (\eta^*)^2$ holds for all $t>t^*$, so that we have
\[
 \eta_t = \sqrt{\frac{\log N}{d_t}} \le \sqrt{\frac{\log N}{d^*}} = \eta^*.
\]
After round $t^*$, the remaining number of times that our algorithm updates the learning rate is $\sum_{t=t^*+1}^T \II{2(1-2c_t)\le\eta_t}$, which can be bounded as 
\[
\sum_{t=t^*+1}^T \II{2(1-2c_t)\le\eta_t} \le \sum_{t=t^*+1}^T \II{2(1-2c_t)\le\eta^*} \le \sum_{t=1}^T \II{2(1-2c_t)\le\eta^*} = \tau^*,
\]
where the first inequality uses the monotonicity of each of the summands in $\eta$ and the second step adds a number of non-negative terms to the sum. The final step is the definition of $\tau^*$. Altogether, this implies the bound $d_T \le d^* + 1 
+ \tau^* \le 9\log N/(\eta^*)^2 + 1$, so the cumulative mixability gap is bounded as
\[
 \sum_{t=1}^T \pa{\EEb{\hloss_t} - \tloss_t} \le \frac 14 \sqrt{d_T \log N} \le \frac{3\log N}{4\eta^*} + \frac{\sqrt{\log N}}{4}.
\]
On the other hand, we have 
\[
 \frac{\log N}{\eta_{T+1}} = \sqrt{d_T\log N} \le \frac{3\log N}{\eta^*} + \sqrt{\log N}.
\]
Thus, the total regret of our algorithm is bounded as 
\[
 \frac{15\log N}{4\eta^*} + \frac{5\sqrt{\log N}}{4}.
\]
Comparing this bound with the lower bound on $R_T^*$ concludes the proof.
$\qed$
\subsection*{Proof of Theorem \ref{thm:lower}}
Our trick will be to construct a randomized prediction algorithm in the batch setup using an algorithm $\mathcal{A}$ for our main setting through the online-to-batch conversion described in Lemma~\ref{lem:onlinetobatch}. First, we fix $B = \frac{1}{2}$ and observe that it follows from the proof of Lemma \ref{ellqloss} that the class of functions $G$ used in the lower bound is a cube of $\{\pm\frac{1}{2}\}$-valued functions consisting of $\lfloor \log_2 N\rfloor$ unique functions and that we are only interested in distributions with $Y \in \left\{\pm \frac{1}{2}\right\}$. For the exact details implying these two facts we refer to the  definition of the hypercube of probability distributions (Definition 8.1  in \citealp{audibert2009fast}) and the first lines of the proof of Theorem 8.3 and 8.4 where \citeauthor{audibert2009fast} sets $h_1 = B, h_2 = -B$. Finally, his proof uses $p_+ = 1, p_- = 0$ and this implies that the $\min\limits_{g}\EE{|g(X) - Y|^q}$ among all measurable functions is achieved by some $\{\pm\frac{1}{2}\}$-valued function for each distribution in the hypercube of $\lfloor \log_2 N\rfloor$ distributions.

With these important observations in mind we will construct the following game of prediction with expert advice. For each round $t=1,2,\dots,T$, the environment draws $(X_t,Y_t)\sim P$, sets $y_t = Y_t+\frac{1}{2}$, and each expert $i\in[N]$ (we can naturally identify $\ev{g_1,g_2,\dots,g_N}$ with experts) predicts $y_{t,i} = g_i(X_t) + \frac 12$. 

This setup allows us to use our online learning algorithm $\mathcal{A}$ to produce a randomized hypothesis $\hat{g}$ defined on $\mathcal{X}$ and taking its values in $\{\pm\frac{1}{2}, 0\}$ and matching the bound of Lemma~\ref{lem:onlinetobatch}. For each $t=1,2,\dots,T$ matching the conditions of Lemma~\ref{lem:onlinetobatch} we run $\A$ on $(X_1,Y_1),\dots,(X_{t-1},Y_{t-1})$. For any $x\in\mathcal{X}$, $\hat{g}$ produces its prediction by feeding the expert predictions $y'_{t,i} = g_i(x)$ to $\mathcal{A}$, obtaining a random prediction $\wh{y}_T\in\ev{0,*,1}$, and finally mapping $\wh{y}_T$ to an element of $\mathcal{Y}$ according to $0\mapsto -\frac 12$, $1\mapsto \frac 12$ and $*\mapsto 0$. Observe that since $Y \in \{\pm\frac{1}{2}\}$ the following holds if we fix $q = \log_2 \frac{1}{c}$,
\[
|Y - \hat{g}(x)|^q = \begin{cases} \II{\hat{g}(x) \neq Y}, & \mbox{if } \hat{g}(x) = \pm \frac{1}{2}; \\ c, & \mbox{if } \hat{g}(x) = 0. \end{cases}
\]
At the same time, for any $g \in G$, it holds that $|Y - \hat{g}(x)|^q = \II{\hat{g}(x) \neq Y}$. To conclude, we have in this setup that the loss of $\hat{g}$ is binary whenever the output is $\pm \frac{1}{2}$ or equal to $c$ whenever the output is equal to zero regardless of the value of $Y \in \{\pm\frac{1}{2}\}$. Using Lemma~\ref{lem:onlinetobatch} we have
\begin{equation}
\label{eq:riskrelations}
\EE{|Y - \hat{g}(X)|^q} = \frac{1}{T}\EE{\sum_{t=1}^T \wh{\ell}_t}\quad \text{and}\quad \min\limits_{g \in G}\EE{|Y - g(X)|^q} = \min\limits_{g \in G}\EE{\II{g(X) \neq Y}}.
\end{equation}
Now, provided that $c \ge \frac{1}{4}$ for our choice $q = \log_2 \frac{1}{c}$ we have $q \le 2$ and
\[
\frac{1}{2} - \frac{1}{2}(q - 1)\log 2 \le \frac{1}{2^q} \le \frac{1}{2} - \frac{1}{4}(q - 1)\log 2,
\]
which implies 
\begin{equation}
\label{eq:cqrelation}
\frac{1}{4}(q - 1)\log 2 \le \frac{1}{2} - c \le \frac{1}{2}(q - 1)\log 2.
\end{equation}
Since $T \ge 4\log N$ in order to apply Lemma \ref{ellqloss} we need $q \ge 1 + \sqrt{\frac{\lfloor\log_2 N\rfloor}{4T}}$ which holds if
\begin{equation}
\label{eq:condition}
\frac{1}{2} - c \ge \log 2\sqrt{\frac{\lfloor\log_2 N\rfloor}{4T}}.
\end{equation}
This corresponds to the regime where the abstentions can give some gain compared to the slow rate bound $O(\sqrt{T\log N})$. Therefore, provided that $c \ge \frac{1}{4}$ since the lower bound of Lemma \ref{ellqloss} holds for 
any $[-\frac{1}{2}, \frac{1}{2}]$-valued estimator (and therefore, for $\{\pm \frac{1}{2}, 0\}$)-valued $\hat{g}$) we have for some $C > 0$,
\[
 \EE{\sum_{t=1}^T \wh{\ell}_t} \ge T\min\limits_{g \in G}\EE{\II{Y\neq g(X)}}+ C\frac{\log N}{1 - 2c},
\]
where we used \eqref{eq:riskrelations}. We remark that the bound in Lemma \ref{lem:audibert} has the factor $\frac{q}{90(q - 1)}\lor e^{-1}$ with the first term being dominant if $1 \le q \le 1.031$. However, it can be easily seen that the values of $q > 1.031$ can be controlled by the choice of the absolute constant $C$ above. Finally, by Jensen's inequality
\[
T\min\limits_{g \in G}\EE{\II{Y\neq g(X)}} = \min\limits_{g \in G}\sum\limits_{t = 1}^T\EE{\II{Y_t\neq g(X_t)}} \ge \EE{\min\limits_{g \in G}\sum\limits_{t = 1}^T\II{Y_t\neq g(X_t)}} = \EE{\min\limits_{i}\sum_{t = 1}^T \ell_{t, i}}.
\]
This implies that $T\min\limits_{g \in G}\EE{\II{Y \neq g(X)}}$ is greater than the expected regret of the best expert.
Therefore, we show that for any randomized algorithm with abstentions there is a random environment such that the expected regret is $\Omega\left(\frac{\log N}{1 - 2c}\right)$ provided that $q = \log_2 \frac{1}{c}$ satisfies $q - 1 \ge\sqrt{\frac{\lfloor\log_2 N\rfloor}{4T}}$. This, of course, implies the existence of the deterministic strategy for the environment.

To finish the proof we need to consider the regime 
\[
1 \le q \le 1 + \sqrt{\frac{\lfloor\log_2 N\rfloor}{4T}},
\]
and show that it gives the lower bound of order $\Omega\left(\sqrt{T\log N}\right)$ which does not depend of $c$ anymore. Since the total loss of any algorithm grows monotonically as $c$ approaches $\frac{1}{2}$, we can prove the lower bound only for $q = 1 + \sqrt{\frac{\lfloor\log_2 N\rfloor}{4T}}$, where as before $q = \log_2\frac{1}{c}$. Indeed, the lower bound will extend automatically for smaller values of $q$ (this corresponds to larger values of $c$). Fortunately, that value of $q$ is covered by the first part of the proof. It is left to observe that in this case we have that $\frac{\log N}{1 - 2c}$ is controlled by $\sqrt{T\log N}$ up to a multiplicative constant factor. The claim follows.
$\qed$
\subsection*{Proof of Theorem \ref{thm:multiclass}}
We begin by defining the misclassification probability 
\[r_t = 
\sum_{k=1}^K p_{t,k} \II{k\neq y_{t}} = \sum_{i=1}^N q_{t,i} \II{y_{t,i}\neq y_t}
\]
and  
noticing that the mix loss can be again written as
\begin{align*}
 -\frac{1}{\eta}\log\sum_i q_{t,i} e^{-\eta \ell_{t,i}} 
 &= -\frac{1}{\eta}\log \pa{1 + r_t\pa{e^{-\eta} - 1}}.
\end{align*}
On the other hand, the expected loss of our algorithm is
\begin{align*}
 \hloss_t = \ab_t c + (1-\ab_t) \II{k_t^* \neq y_t}.
\end{align*}
Notice that this equals $c$ when $p_t^* < \frac 12$. In this case, the error probability satisfies 
$r_t \ge \frac 12$ so that the mix loss is can be lower bounded as

\begin{align*}
-\frac{1}{\eta}\log \pa{1 + r_t\pa{e^{-\eta} - 1}} &\ge 
-\frac{1}{\eta}\log \pa{\frac 12 \pa{1 + e^{-\eta}}} \ge -\frac{1}{\eta}\log \pa{1- \frac \eta2 + \frac {\eta^2}{4}} \ge \frac 12 - \frac \eta4
\\
&\ge \frac 12 - \frac{2(1 - 2c)}{4} = c,
\end{align*}
where in the first line we used the inequalities $e^{-x}\le 1-x + \frac{x^2}{2}$ and $\log(1-x)\le -x$ that both hold for $x\ge 0$, and our assumption that $\eta\le 2(1-2c)$ in the last line.
Thus, the expected loss lower bounds the mix loss in this case.

Considering the case where $p_t^*\ge \frac 12$, we have $r_t = 1 - p_t^* \le \frac 12$ whenever 
$k_t^* = y_t$ holds, and $\frac 12 \le p_t^*\le r_t$ otherwise (since in this case we have $r_t = 
\sum_{k\neq y_t} p_{k,t} \ge p^*_t$). In the first case, we have
\begin{align*}
 \EEb{\hloss_t} = \ab_t c = 2(1-p_t^*) c = 2r_tc = r_t + (2c - 1)r_t,
\end{align*}
while in the second one we have
\begin{align*}
 \EEb{\hloss_t} = \ab_t (c-1) + 1 = 2(p_t^*-1) (1-c) + 1 \le 2(r_t - 1) (1-c) + 1 = r_t + (2c - 1)(1 
- 
r_t).
\end{align*}
Summarizing both cases, we have
\[
 \EEb{\hloss_t}\le r_t + \pa{2c - 1} (r_t \wedge (1-r_t)).
\]
The proof is concluded by noticing that this function is upper-bounded by the mix loss $\tloss_t$ by 
the same argument as used in the proof of Lemma~\ref{lem:mixbound}, and using 
Lemma~\ref{lem:mixregret} to upper-bound the cumulative mix loss.
$\qed$
\subsection*{Proof of Corollary \ref{cor:inflclass}}
Our key technical tool is the following lemma which can be seen as an online learning analog of the renowned lemma by \cite{Vapnik68}. 

\begin{lemma}[\cite{ben2009agnostic}]
\label{lem:bendavid}
Fix $T$ and assume that $\mathcal{H}$ has the Littlestone dimension $L$. There is a set $\mathcal{H}_{T}$ of at most $\sum\limits_{i = 0}^L{T \choose i}$ experts such that for any $h \in \mathcal{H}$ there is at least one expert in $\mathcal{H}_{T}$ predicting as $h$ on any sequence of length at most $T$.
\end{lemma}
Now the proof follows immediately from Corollary \ref{cor:main}, Lemma \ref{lem:bendavid} and $\sum\limits_{i = 0}^L{T \choose i} \le \left(\frac{eT}{L}\right)^L$ whenever $T \ge L$. We remark that the results of Section \ref{sec:changingcosts} can also be straightforwardly adapted for the sets of experts having finite Littlestone dimension.
$\qed$
\end{document}